%% file: main.tex
\RequirePackage{amsmath}
\documentclass[runningheads]{llncs}
\usepackage{graphicx}
\usepackage{hyperref}

\input{preamble.tex}

\begin{document}

\title{On the Linguistic Capacity of Real-Time Counter Automata}


\author{William Merrill\\\email{willm@allenai.org}}
%
%
\institute{Allen Institute for AI, Seattle WA 98103}
\maketitle

\input{abstract}
\input{intro}
\input{definitions}
\input{hierarchy}
\input{closure}
\input{compositionality}
\input{semilinearity}
\input{conclusion}

\section*{Acknowledgments}
Thanks to Dana Angluin, Robert Frank, Yiding Hao, Roy Schwartz, and Yoav Goldberg, as well as other members of Computational Linguistics at Yale and the Allen Institute for AI, for their suggestions on various versions of this work. Additional thanks to several anonymous reviewers for their exceptional feedback.

%

\bibliographystyle{splncs04}
\bibliography{main}

\end{document}

%% file: preamble.tex
\usepackage{amssymb}
\usepackage{mathtools}
\usepackage{bbm}

\usepackage{algorithm}
\usepackage[noend]{algpseudocode}
\algnewcommand{\IfThen}[2]{\State \algorithmicif\ #1\ \algorithmicthen\ #2}
\algnewcommand\algorithmicforeach{\textbf{for each}}
\algdef{S}[FOR]{ForEach}[1]{\algorithmicforeach\ #1\ \algorithmicdo}

\usepackage{pgf}
\usepackage{tikz}
\usetikzlibrary{arrows,automata}

\newcommand{\bm}[1]{\mathbf{#1}}  
\newcommand{\bv}[1]{\mathbf{#1}}  
\newcommand{\computes}[1]{\rightarrow_{#1}}
\newcommand{\den[2]}{\mbox{ $[\![ #2 ]\!]$}}

\DeclarePairedDelimiter\abs{\lvert}{\rvert}%
\DeclarePairedDelimiter\norm{\lVert}{\rVert}%

\makeatletter
\let\oldabs\abs
\def\abs{\@ifstar{\oldabs}{\oldabs*}}
\let\oldnorm\norm
\def\norm{\@ifstar{\oldnorm}{\oldnorm*}}
\makeatother

\newcommand{\CL}{\mathrm{CL}}
\newcommand{\SCL}{\mathrm{SCL}}

\newcommand{\ICL}{\mathrm{ICL}}
\newcommand{\nQCL}{\mathrm{\tilde QCL}}
\newcommand{\nQSCL}{\mathrm{\tilde QSCL}}

\newcommand{\TCL}{\mathrm{\Theta CL}}

\DeclarePairedDelimiter\floor{\lfloor}{\rfloor}

\newcommand{\sgn}{\mathrm{sgn}}
\newcommand{\thresh}{\mathbbm{1}_+}
\newcommand{\tup}[1]{\langle #1 \rangle}

%% file: abstract.tex
\begin{abstract}
Counter machines have achieved a newfound relevance to the field of natural language processing (NLP): recent work suggests some strong-performing recurrent neural networks utilize their memory as counters. Thus, one potential way to understand the success of these networks is to revisit the theory of counter computation. Therefore, we study the abilities of real-time counter machines as formal grammars, focusing on formal properties that are relevant for NLP models. We first show that several variants of the counter machine converge to express the same class of formal languages. We also prove that counter languages are closed under complement, union, intersection, and many other common set operations. Next, we show that counter machines cannot evaluate boolean expressions, even though they can weakly validate their syntax. This has implications for the interpretability and evaluation of neural network systems: successfully matching syntactic patterns does not guarantee that counter memory accurately encodes compositional semantics. Finally, we consider whether counter languages are semilinear. This work makes general contributions to the theory of formal languages that are of potential interest for understanding recurrent neural networks.
\end{abstract}

%% file: intro.tex
\section{Introduction}


It is often taken for granted that modeling natural language syntax well requires a grammar formalism sensitive to compositional structure. Early work in linguistics established that finite-state models are insufficient for describing the dependencies in natural language data \cite{chomsky1956three}. Instead, a formalism capable of expressing the relations in terms of hierarchical constituents ought to be necessary.

Recent advances in deep learning and NLP, however, challenge this long-held belief. Neural network formalisms like the long short-term memory network (LSTM) \cite{hochreiter-1997-lstm} perform fairly well on tasks requiring structure sensitivity \cite{linzen2016assessing}, even though it is not obvious that they have the capacity or bias to represent hierarchy. This mismatch raises interesting questions for both linguists and practitioners of NLP. It is unclear what about the LSTM architecture might lend itself towards good linguistic representations, and under what conditions these representations might fall short of grasping the structure and meaning of language.

Recent work has suggested that the practical learnable capacity of LSTMs resembles that of counter machines \cite{merrill-2019-sequential}\cite{suzgun-etal-2019-lstm} \cite{weiss2018}. Theoretically, this connection is motivated by studying the ``saturated" version \cite{merrill-2019-sequential} of the LSTM network, i.e. replacing each continuous activation function with a step function. Under these conditions, the LSTM reduces to a discrete automaton that uses its memory cell as integer-valued counter registers. \cite{weiss2018} define a simplified class of counter languages that falls within the expressive capacity of this saturated LSTM model. On the other hand, a more general class of counter languages is an upper bound on the expressive capacity of saturated LSTMs \cite{merrill-2019-sequential}. Thus, there is a strong theoretical connection between LSTMs and counter automata.

Furthermore, these theoretical results for saturated LSTMs seem to predict what classes of formal languages LSTMs can empirically learn. \cite{weiss2018} show how LSTMs learn to model languages like $a^nb^n$ by using their memory to count $n$, whereas other recurrent neural network architectures without saturated counting abilities fail. Similarly, \cite{merrill-2019-sequential} shows how LSTMs cannot reverse strings, just like real-time counter automata \cite{fischer1968counter}. Further, LSTMs can flawlessly model $1$-Dyck strings by using their memory to count \cite{suzgun-etal-2019-lstm}, but, like counter automata, they cannot model $2$-Dyck \cite{suzgun2019memory}. It seems that, where LSTMs succeed at algorithmic tasks, they do so by counting, and where they fail, their failure might be explained by their inability to reliably implement more complex types of memory.

Inspired by the connection of LSTMs to counter automata, we study the formal properties of counter machines as language recognizers. We do this with the hope of understanding the abilities of counter-structured memory, and to what degree it has computational properties well-suited for representing compositional structure. The contributions of this paper are as follows:

\begin{itemize}
    \item We prove that several interesting counter machine variants converge to the same linguistic capacity, whereas simplified counter machines \cite{weiss2018} are strictly weaker than classical counter machines.
    \item We demonstrate that counter languages are closed under complement, union, intersection, and many other common operations.
    \item We show counter machines cannot evaluate compositional boolean expressions, even though they can check whether such expressions are well-formed.
    \item We prove that a certain subclass of the counter languages are semilinear, and conjecture that this result holds for all counter languages.
\end{itemize}

%% file: definitions.tex
\section{Definitions} \label{sec:counter-machines}

Informally, we can think of counter automata as finite-state automata that have been augmented by a finite number of integer-valued counters. While processing a string, the machine can update the values of the counters, and the counters can in turn inform the machine's state transitions.

Early results in theoretical computer science established that a $2$-counter machine with unbounded computation time is Turing-complete \cite{fischer1966turing}. However, restricting computation to be real-time (i.e. one iteration of computation per input) severely limits the counter machine's computational capacity \cite{fischer1968counter}. A similar fact holds for recurrent neural networks like LSTMs \cite{weiss2018}. We study the language recognition abilities of several types of real-time counter automata.

\subsection{General Counter Machines}

The first counter automaton we introduce is the \textit{general counter machine}. This machine manipulates its counters by adding or subtracting from them. Later, we define other variants of this general automaton. For $m \in \mathbb{Z}$, let ${\pm}m$ denote the function $\lambda x. x \pm m$. Let $\times 0$ denote the constant zero function $\lambda x. 0$.

\begin{definition}[General counter machine \cite{fischer1968counter}]
A $k$-counter machine is a tuple $\langle \Sigma, Q, q_0, u, \delta, F \rangle$ with
\begin{enumerate}
    \item A finite alphabet $\Sigma$
    \item A finite set of states $Q$\footnote{The original definition \cite{fischer1968counter} distinguishes between ``autonomous" and ``polling" states, a distinction that is vacuous in the real-time case we are studying.}
    \item An initial state $q_0$
    \item A counter update function
    \begin{equation*}
        u : \Sigma \times Q \times \{0, 1\}^k \rightarrow
        \big( \{+m : m \in \mathbb{Z} \} \cup \{ \times 0 \} \big)^k
    \end{equation*}
    \item A state transition function
    \begin{equation*}
        \delta : \Sigma \times Q \times \{0, 1\}^k \rightarrow Q
    \end{equation*}
    \item An acceptance mask
    \begin{equation*}
        F \subseteq Q \times \{0, 1\}^k
    \end{equation*}
\end{enumerate}
\end{definition}

A machine processes an input string $x$ one token at a time.
For each token, we use $u$ to update the counters and $\delta$ to update the state according to the current input token, the current state, and a finite mask of the current counter values. We formalize this in \autoref{def:counter-machine-computation}.

For a vector $\bv v$, let $z(\bv v)$ to denote the broadcasted ``zero-check" function, i.e.
\begin{equation}
z(\bv v)_i =
\begin{cases}
0 & \textrm{if} \; v_i = 0 \\
1 & \textrm{otherwise.}
\end{cases}
\end{equation}

\begin{definition}[Counter machine computation] \label{def:counter-machine-computation}
Let $\langle q, \bv c\rangle \in Q \times \mathbb{Z}^k$ be a configuration of machine $M$. Upon reading input $x_t \in \Sigma$, we define the transition
\begin{equation*}
\langle q, \bv c\rangle \computes{x_t} \langle
\delta(x_t, q, z(\bv c)) ,
u(x_t, q, z(\bv c))(\bv c)
\rangle.
\end{equation*}
\end{definition}

\begin{definition}[Real-time acceptance]
For any string $x \in \Sigma^*$ with length $n$, a counter machine accepts $x$ if there exist states $q_1, .., q_n$ and counter configurations $\bv c_1, .., \bv c_n$ such that
\begin{equation*}
    \langle q_0, \bv 0 \rangle \computes{x_1} \langle q_1, \bv c_1 \rangle \computes{x_2} .. \computes{x_n} \langle q_n, \bv c_n \rangle \in F .
\end{equation*}
\end{definition}

\begin{definition}[Real-time language acceptance]
A counter machines accepts a language $L$ if, for each $x \in \Sigma^*$, it accepts $x$ iff $x \in L$.
\end{definition}

We denote the set of languages acceptable in real time by a general counter machine as $\CL$. We will use the terms ``accept" and ``decide" interchangeably, as accepting and deciding a language are equivalent for real-time automata.

\begin{figure}[ht!]
    \centering
    \begin{tikzpicture}[->,>=stealth',shorten >=1pt,auto,node distance=2.8cm, semithick]
    
    \node[initial,state] (0) {$q_0$};
    \node[state] (1) [right of=0] {$q_1$};
    \node[state] (2) [right of=1] {$q_2$};

    \path (0) edge [loop above] node {$a/{+}1$} (0)
              edge [above] node {$b/{-}1$} (1);

    \path (1) edge [loop above] node {$b/{-}1$} (1)
              edge [above] node {$a/{+}0$} (2);
    
    \path (2) edge [loop above] node {$a,b/{+}0$} (2);
    
    \end{tikzpicture}
    \caption{A graphical representation of a $1$-counter machine that accepts $\{a^nb^n \mid n \in \mathbb{N} \}$ if we set $F$ to verify that the counter is $0$ and we are in either $q_0$ or $q_1$.}
    \label{fig:example}
\end{figure}
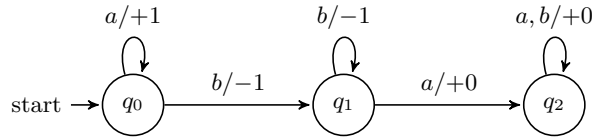

\begin{figure}[ht!]
    \centering
    \begin{equation*}
        \tup{0, q_0} \computes{a} \tup{1, q_0} \computes{a} \tup{2, q_0} \computes{b} \tup{1, q_1} \computes{b} \tup{0, q_0} \in F
    \end{equation*}
    \begin{equation*}
        \tup{0, q_0} \computes{a} \tup{1, q_0} \computes{a} \tup{2, q_0} \computes{b} \tup{1, q_1} \computes{a} \tup{1, q_2} \notin F
    \end{equation*}
    \caption{Behavior of the counter machine in \autoref{fig:example} on $aabb$ (\textbf{top}) and $aaba$ (\textbf{bottom}).}
    \label{fig:example-computation}
\end{figure}

Unlike context-free (CF) grammars, general counter machines cannot accept palindromes \cite{fischer1968counter}. However, they can accept non-CF languages like $a^nb^nc^nd^n$ \cite{fischer1968counter}. Thus, $\CL$ does not fall neatly into the classical Chomsky hierarchy.

\subsection{Restricted Counter Machines}

Now, we can can consider various restrictions of the general counter machine, and the corresponding classes of languages acceptable by such automata.

First, we present the \textit{simplified counter machine} \cite{weiss2018}. The counter update function in the simplified counter machine has two important constraints compared to the general machine. First, it can only be conditioned by the input symbol at each time step. Second, it can only increment or decrement its counters instead of being able to add or subtract arbitrary constants.

\begin{definition}[Simplified counter machine] \label{def:scm}
A counter machine is simplified if $u$ has the form
\begin{equation*}
    u : \Sigma \rightarrow \{ -1, \; +0, \; +1, \; \times 0 \}^k .
\end{equation*} 
\end{definition}

Another variant that we consider is the \textit{incremental counter machine}. The arguments to the update function of this machine are not restricted, but the additive operations are constrained to ${\pm}1$.

\begin{definition}[Incremental counter machine] \label{def:icm}
An counter machine is incremental if $u$ has the form
\begin{equation*}
    u : \Sigma \times Q \times \{0, 1\}^k \rightarrow \{ -1, \; +0, \; +1, \; \times 0 \}^k .
\end{equation*}
\end{definition}

Finally, we define a \textit{stateless} variant of the counter machine. Removing state from the counter machine is equivalent to allowing it to only have one state $q_0$.

\begin{definition}[Stateless counter machine] \label{def:qcm}
A counter machine is stateless if $Q=\{q_0\}$.
\end{definition}

\subsection{Saturated LSTMs}

The LSTM is a recurrent neural network resembling a counter machine. At each step, a vector encoding of the input $\bv x_t$ is used to update the state vectors $\bv c_t, \bv h_t$ and produce an acceptance decision $y_t$. Let $\thresh$ denote the function that returns $1$ for positive reals and $0$ otherwise. Similarly, let $\sgn$ return $1$ for positive reals and $-1$ otherwise. Let $\odot$ be elementwise multiplication over vectors. The saturated LSTM's recurrent update \cite{merrill-2019-sequential}, parameterized by weight tensors $\bm W$ and $\bv b$, is:

\begin{align}
    \bv f_t &= \thresh(\bm W^f \bv x_t + \bm U^f \bv h_{t-1}) \\
    \bv i_t &= \thresh(\bm W^i \bv x_t + \bm U^i \bv h_{t-1}) \\
    \bv o_t &= \thresh(\bm W^o \bv x_t + \bm U^o \bv h_{t-1}) \\
    \bv{\tilde c_t} &= \sgn(\bm W^c \bv x_t + \bm U^c \bv h_{t-1}) \\
    \bv c_t &= \bv f_t \odot \bv c_{t-1} + \bv i_t \odot \bv{\tilde c_t} \\
    \bv h_t &= \bv o_t \odot \bv c_t \label{eq:nonlin} \\
    y_t &= \thresh(\bv w^y \cdot \bv h_t + b^y) .
\end{align}

\noindent We say the LSTM accepts iff $y_t = 1$. In practice, \eqref{eq:nonlin} is often $\bv o_t \odot \tanh(\bv c_t)$. We remove the $\tanh$ for clarity, as its monotonicity does not change the expressiveness of the saturated network. These equations specify a discrete automaton that is highly similar to a counter machine \cite{weiss2018}\cite{merrill-2019-sequential}. 

The major difference between the saturated LSTM and the classical counter machines is that the LSTM partitions the counter values by passing them through a linear map and applying a thresholding function, whereas the classical counter machines probes whether or not the counters are zero. For example, for a counter $c$, the saturated LSTM could test $c \leq 5$, whereas the general counter machine can only test $c = 0$. Motivated by this, we define the \textit{threshold counter machine}, which views its counters by thresholding them instead of testing equality to $0$.

\begin{definition}[Threshold counter machine]
A threshold counter machine is a general counter machine where all occurrences of the zero-check function $z$ are redefined as predicates of the form $\lambda m. c \leq m$ for $m \in \mathbb{Z}$. We refer to such a function by the shorthand ${\leq}m$.
\end{definition}

%% file: hierarchy.tex
\section{Counter Language Hierarchy} \label{sec:counter-hierarchy}

\subsection{Simplified Counter Languages}

Our first result relating counter classes is to show that the simplified counter languages are a proper subset of the general counter languages. The weakness of the simplified machine is that the update function is conditioned only by the input symbol. Thus, languages like $a^mb^{2m}$, which require switching counting behavior, cannot be decided correctly. We formalize this in \autoref{thm:scl-weakness}.

\begin{theorem}[Weakness of $\SCL$] \label{thm:scl-weakness}
Let $\SCL$ be the set of languages acceptable in real time by a simplified counter machine. Then $\SCL \subsetneq \CL$.
\end{theorem}

\begin{proof}
Consider the language $a^mb^{2m}$. This is trivially acceptable by a 1-counter machine that adds 2 for each $a$ and subtracts 1 for each $b$. On the other hand, we shall show that it cannot be accepted by any simplified machine. Assume by way of contradiction that such a simplified machine $M$ exists. We assume without loss of generality that $M$ does not apply a ${\times}0$ update, as doing so would erase all information about the prefix.

Tracking the ratio between $a$'s and $b$'s requires infinite state. Thus, the counters of $M$, as opposed to the finite state, must encode whether $2m = l$ for strings of the form $a^mb^l$. Let $c$ be the value of some counter in $M$. We can decompose $c$ into the update contributed by $a$'s and the the update contributed by $b$'s:
\begin{align}
    c &= mu_a + lu_b , \\
    u_a, u_b &\in \{-1, 0, 1\} .
\end{align}

Exhausting all the possible functions that $c$ can compute, we get
\begin{align}
    c &\in \{ 0, \pm m, \pm l, \pm (m + l), \pm (m - l) \} \\
    z(c) &\in \{0, \mathbbm{1}_{m>0}, \mathbbm{1}_{l>0}, \mathbbm{1}_{m+l>0}, \mathbbm{1}_{m-l \neq 0} \} .
\end{align}
We ignore the first four options for $z(c)$, as they do not relate $m$ to $l$. The final option tests $m/l = 1$, not $2$. Thus, $z(c)$ cannot test whether $2m = l$.
\end{proof}

Note that this argument breaks down if the counter update can depend on the state. In that case, we can build a machine that has two counters and two states: $q_0$ adds 1 to the first counter while it reads $a$, and then decrements the first counter and increments the second counter when it reads $b$. When the first counter is empty and the second counter is not empty, $q_0$ transitions to $q_1$, which decrements the second counter. We accept iff both counters are $0$ after $x_n$.

\subsection{Incremental Counter Languages}

Unlike the simplified counter machine, the incremental machine has the same linguistic capacity as the general machine. We can simulate each counter on a general machine with a finite amount of overhead. This provides a reduction from general to incremental machines.

\begin{theorem}[Generality of $\ICL$] \label{thm:generality-icl}
Let $\ICL$ be the set of languages acceptable in real time by an incremental counter machine. Then $\ICL = \CL$.
\end{theorem}

\begin{proof}
Let $d$ be the maximum that is ever added or subtracted from a counter $c$ in $M$. We simulate $c$ in $M'$ using a counter $c'$ and a value $q \in \mathbb{Z} \mod d$ encoded in finite state. We will implement a ``ring-counter" encoding of $c$ such that
\begin{align*}
    c' &= \floor{c / d} \\
    q &= c \mod d .
\end{align*}

To simulate a $\times 0$ update on $c$, we apply $\times 0$ to $c'$, and transition state such that $q := 0$. To simulate a $+m$ update on $c$ for some $m \in \mathbb{Z}$, we first change state such that $q := (q + m) \mod d$. Next, we apply the following update to $c'$:
\begin{equation}
    \begin{cases}
    +1 & \textrm{if} \; q+m \geq d \\
    -1 & \textrm{if} \; q+m < 0 \\
    +0 & \textrm{otherwise.}
    \end{cases}
\end{equation}

\noindent We can compute $z(c)$ by checking whether $z(c') = 0$ and $q=0$.
\end{proof}

\subsection{Stateless Counter Languages}

Similarly, restricting a counter machine to be stateless does not weaken its expressive capacity. We show how to reduce an arbitrary stateful machine to a stateless machine that has been augmented with additional counters. The key idea here is that we can use the additional counters as a one-hot vector that tracks the state of the original machine.

\begin{theorem}[Generality of $\nQCL$] \label{thm:nqcl-generality}
Let $\nQCL$ be the set of languages acceptable in real time by a stateless counter machine. Then $\nQCL = \CL$.
\end{theorem}

\begin{proof}
We define a new stateless machine $M'$ to simulate $M$ by adding a $|Q|$-length vector of counters called $\bv q'$. Let $\bv \omega(i)$ denote the $\abs{Q}$-length one-hot vector encoding $i$, i.e. $[\bv \omega(i)]_i = 1$, and all other indices are $0$. We consider $\bv \omega(0) = \bv 0$.

At initialization, $\bv q'$ encodes the initial state since $\bv q' = \bv 0 = \bv \omega(0)$. Furthermore, we define the invariant that, at any given time, $\bv q' = \bv \omega(i)$ for some state $i$. Thus, the additional counters now encode the current state.

Let $\bv x \Vert \bv y$ denote the concatenation of vectors $\bv x$ and $\bv y$. We define the new acceptance mask in $M'$ as
\begin{equation} \label{eq:stateless-f}
    F' = \{ \langle q_0, \bv b \Vert \bv \omega(i) \rangle \mid \langle q_i, \bv b \rangle \in F \} .
\end{equation}

\noindent We can update the counters inherited from $M$ analogously to \eqref{eq:stateless-f}. The last step is to properly update the state counters $\bv q'$. For each transition $\delta(x_t, q_i, \bv b) = q_j$ in $M$, we update $\bv q'$ by adding $-\bv \omega(i) + \bv \omega(j)$. This ensures $\bv q'$ is correct since
\begin{equation}
\bv \omega(i) + \big( -\bv \omega(i) + \bv \omega(j) \big) = \bv \omega(j) .
\end{equation}
\end{proof}

\subsection{Threshold Counter Languages}

We show that the threshold counter languages are equivalent to the general counter languages. As thresholding is a key capability of the saturated LSTM formalism, this suggests that much of the LSTM capacity falls within the general counter languages, although it does not provably establish containment.

\begin{theorem}[Generality of $\TCL$] \label{thm:generality-tcl}
Let $\TCL$ be the languages acceptable in real time by a threshold counter machine. Then $\TCL = \CL$.
\end{theorem}

\begin{proof}
Given the ability to check ${\leq}m$ on the counters for any $m$, we can simulate ${=}0$ by checking both ${\leq}{-}1$ and ${\leq}0$. Thus, $\CL \subseteq \TCL$. To prove the other direction, we show how to simulate applying ${\leq}m$ to the counters using only ${=}0$.

Assume without loss of generality that only one threshold check $m$ applies to each counter $c$ (we can create copies of a counter and distribute the threshold checks over them if this is not the case), and that $m > 0$. We implement a ring-counter construction similar to the one used in \autoref{thm:generality-icl}, representing $c$ with a new counter $c' = \floor{c / m}$ and finite-state component $q = c \mod m$. We also store the sign of $c$ in finite state by recording whenever both $c'$ and $q$ pass zero. Having all this information, we conclude $c \leq m$ iff the sign is negative or $c' = 0$.
\end{proof}

The construction in \autoref{thm:generality-tcl} can be directly adapted to show that a general counter machine can simulate checking ${=}m$ in addition to ${=}0$.

\subsection{Summary}

The general counter machine, incremental counter machine, stateless counter machine, and threshold counter machine all converge to the same linguistic capacity, which we call $\CL$. The simplified counter machine \cite{weiss2018}, however, has a linguistic capacity $\SCL$ that is strictly weaker than $\CL$.

%% file: closure.tex
\section{Closure Properties} \label{sec:closure-properties}

Another way to understand the counter languages is through their closure properties. It turns out that the real-time counter languages are closed under a wide array of common operations, including complement, intersection, union, set difference, and symmetric set difference. The general result in \autoref{thm:binary_set_operation_closure} implies these closure properties, as well as many others.

\begin{theorem}[General set operation closure] \label{thm:binary_set_operation_closure}
Let $P$ be an $m$-ary operation over languages. If there exists an $m$-ary boolean function $p$ such that
\begin{equation*}
\mathbbm{1}_{P(L_1, .., L_m)}(x) = p \big( \mathbbm{1}_{L_1}(x), .., \mathbbm{1}_{L_m}(x) \big) ,
\end{equation*}
\noindent then $\CL$ and $\SCL$ are both closed under $P$.
\end{theorem}

\begin{proof}
First, we construct counter machines $M_1, .., M_m$ that decide the counter languages $L_1, .., L_m$. We define a new machine $M'$ that, on input $x$, simulates $M_1, .., M_m$ in parallel, and accepts if and only if
\begin{equation}
    p(M_1(x), .., M_m(x)) = 1 .
\end{equation}
\end{proof}

\subsubsection{Corollaries.} Let $\Lambda$ be a placeholder for either $\CL$ or $\SCL$. Let $L_1, L_2 \in \Lambda$. By \autoref{thm:binary_set_operation_closure}, $\Lambda$ is closed under the following operations:
\begin{gather}
    \Sigma^* \setminus L_1 \\
    L_1 \cap L_2  \\
    L_1 \cup L_2  \\
    L_1 \setminus L_2  \\
    (L_1 \setminus L_2) \cup (L_2 \setminus L_1)  .
\end{gather}

%% file: compositionality.tex
\section{Compositional Expressions} \label{sec:expressions}

We now study the abilities of counter machines on the language $L_m$ (\autoref{def:ln}). Like natural language, $L_m$ has a deep structure consisting of recursively nested hierarchical constituents.

\begin{definition}[$L_m$ \cite{fischer1968counter}] \label{def:ln}
For any $m$, let $L_m$ be the language generated by:
\begin{verbatim}
<exp> -> <VALUE>
<exp> -> <UNARY> <exp>
<exp> -> <BINARY> <exp> <exp>
..
<exp> -> <m-ARY> <exp> .. <exp>
\end{verbatim}
\end{definition}

Surprisingly, even a $1$-counter machines can decide $L_m$ in real time by implementing \autoref{alg:decide} \cite{fischer1968counter}. \autoref{alg:decide} uses a counter to keep track of the depth at any given index. If the depth counter reaches $-1$ at the end of the string, the machine has verified that the string is well-formed. We define the arity of a \texttt{<VALUE>} as $0$, and the arity of an \texttt{<m-ARY>} operation as $m$.

\begin{algorithm}
	\caption{Deciding $L_m$ \cite{fischer1968counter}}
	\label{alg:decide}
	\begin{algorithmic}[1]
		\Procedure{Decide}{$x$}
		    \State $c \gets 0$
		    \ForEach{$x_t \in x$}
		            \State $c \gets c + \textsc{Arity}(x_t) - 1$
		    \EndFor
			\State \Return $c = -1$
		\EndProcedure
	\end{algorithmic}
\end{algorithm}

\subsection{Semantic Evaluation as Structure Sensitivity}

While \autoref{alg:decide} decides $L_m$, it is agnostic to the deep structure of the input in that it does not represent the hierarchical dependencies between tokens. This means that it could not be used to evaluate these expressions. Based on this observation, we prove that no counter machine can evaluate boolean expressions due to the deep structural sensitivity that semantic evaluation (as opposed to syntactic acceptance) requires. We view boolean evaluation as a simpler formal analogy to evaluating the compositional semantics of natural language.

To be more formal, consider an instance of $L_2$ with values $\{0, 1\}$ and binary operations $\{\wedge, \vee\}$. We assign the following semantics to the terminals:
\begin{align}
    \den{0} = 0 &\quad \den{1} = 1 \\
    \den{\wedge} &= \lambda pq. \; p \wedge q \\
    \den{\vee} &= \lambda pq. \; p \vee q .
\end{align}

Our semantics evaluates each nonterminal by applying the denotation of each syntactic argument to the semantic arguments of the operation. For example:
\begin{equation}
    \den{\vee 0 1} = \den{\vee}(\den{0}, \den{1}) = 0 \vee 1 = 1 .
\end{equation}
\noindent We also define semantics for non-constituent prefixes via function composition:
\begin{equation}
    \den{\vee \vee} = \den{\vee} \circ \den{\vee} = \lambda pqr. \; p \vee q \vee r .
\end{equation}

We define the language $B$ as the set of valid strings $x$ where $\den{x} = 1$.

\begin{theorem}[Weak evaluation] \label{thm:weak-evaluation}
$B \notin \CL$.
\end{theorem}

\begin{proof}
Assume by way of contradiction that there exists a counter machine deciding $B$. We consider an input $x$ that contains a prefix of $p$ operators followed by a suffix of $p+1$ values. For the machine to evaluate $x$ correctly, the configuration after $x_p$ must encode which boolean function $x_p$ specifies.

However, a counter machine with $k$ counters only has $O(p^k)$ configurations after reading $p$ characters. We show by induction over $p$ that an $p$-length prefix of operators can encode $\geq 2^p$ boolean functions. Since the machine does not have enough configurations to encode all the possibilities, we reach a contradiction.

\paragraph{Base Case.} With $p=0$, we have a null prefix followed by one value that determines $\den{x}$. We can represent exactly 1 ($2^0$) function, which is the identity.

\paragraph{Inductive Case.} The expression has a prefix of operators $x_{1:p+1}$ followed by values $x_{p+2:2p+3}$. We decompose the semantics of the full expression to

\begin{equation} \label{binaryComposition}
\den{x} = \den{x_1}(\den{x_{2:2p+2}}, \den{x_{2p+3}}) .
\end{equation}

Since $\den{x_{2:2p+2}}$ has a prefix of $p$ operators, we apply the inductive assumption to show it can represent $ \geq 2^p$ boolean functions. Define $f$ as the composition of $\den{x_1}$ with $\den{x_{2:2p+2}}$. There are two possible values for $f$: $f_\wedge$, obtained when $x_1 = \wedge$, and $f_\vee$, obtained when $x_1 = \vee$. We complete the proof by verifying that $f_\wedge$ and $f_\vee$ are necessarily different functions.

To do this, consider the minimal sequence of values that will satisfy them according to a right-to-left ordering of the sequences. For $f_\wedge$, this minimal sequence ends in $1$, whereas for $f_\vee$ it must end in a $0$. Therefore, $f$ represents at least two unique functions for each value of $\den{x_{2:2p+2}}$. Thus, a $p+1$-length sequence of prefixes can encode $\geq 2 \cdot 2^p = 2^{p+1}$ boolean functions.

\end{proof}

\autoref{thm:weak-evaluation} shows how counter machines cannot represent certain hierarchical dependencies, even when the generated language is within the counter machine's weak expressive capacity. This is analogous to how CFGs can weakly generate Dutch center embedding \cite{Pullum1980-PULNLA}, even though they cannot assign the correct cross-serial dependencies between subjects and verbs \cite{bresnan-1982-cross-serial}. Thus, while counter memory can track certain formal properties of compositional languages, it cannot represent the underlying hierarchical structure in a deep way.


%% file: semilinearity.tex
\section{Semilinearity} \label{sec:semilinearity}

Semilinearity is a condition that has been proposed as a desired property for any formalism of natural language syntax \cite{joshi1990convergence}. Intuitively, semilinearity ensures that the set of string lengths in a language is not unnaturally sparse. Regular, context-free, and a variety of mildly context-sensitive languages are known to be semilinear \cite{joshi1990convergence}. The semilinearity of $\CL$ is an interesting open question for understanding the abilities of counter machines as grammars.

\subsection{Definition}

We first define semilinearity over sets of vectors before considering languages. To start, we introduce the notion of a linear set:

\begin{definition}[Linear set]
A set $S \subseteq \mathbb{N}^k$ is linear if there exist $\bm W \in \mathbb{N}^{k\times m}$ and $\bv b \in \mathbb{N}^k$ such that
\begin{equation*}
    S = \left\{ \bv n \in \mathbb{N}^m \mid \bm W \bv n + \bv b = \bv 0 \right\} .
\end{equation*}
\end{definition}

Semilinearity, then, is a weaker condition that specifies that a set is made up of a finite number of linear components:

\begin{definition}[Semilinear set]
A set $S \subseteq \mathbb{N}^k$ is semilinear if it is the finite union of linear sets.
\end{definition}

To apply this definition to a language $L$, we translate each string $x \in L$ into a vector by taking $\Psi(x)$, the Parikh mapping of $x$. The Parikh mapping of a sentence is its ``bag of tokens" representation. For example, the Parikh mapping of $abaa$ with respect to $\Sigma = \{a, b\}$ is $\langle 3, 1 \rangle$. We say that a language $L$ is semilinear if its image under $\Psi$, i.e. $\{\Psi(x) \mid x \in L\}$, is semilinear.

\subsection{Semilinearity of Counter Languages}

We do not prove that the general counter languages are semilinear, but we do prove it for a dramatically restricted subclass of the counter languages. Define $\nQSCL$ as the set of language acceptable by a counter machine that is \textit{both} simplified (\autoref{def:scm}) \textit{and} stateless (\autoref{def:qcm}). $\nQSCL$ is indeed semilinear.

\begin{theorem}[Semilinearity of $\nQSCL$] \label{thm:semilinearity}
For all $L \in \nQSCL$, $L$ is semilinear.
\end{theorem}

\begin{proof}
Applying the definition of counter machine acceptance, we express $L$ as
\begin{align}
L = \bigcup_{\bv b \in F} \{ x \; | \; \bv c(x) = \bv b \}
= \bigcup_{\bv b \in F} \bigcap_{i=1}^k \{ x \; | \; c_i(x) = b_i \}. 
\end{align}

Semilinear languages are closed under finite union and intersection, so we just need to show $\{ x \; | \; c_i(x) = b_i \}$ is semilinear. We apply the following trick:
\begin{equation}
    \{ x \; | \; c_i(x) = b_i \} = \Sigma^* \Vert Z_i \Vert L_i
\end{equation}
\noindent where $Z_i$ is the set of all tokens that set counter $i$ to 0, and $L_i$ is the set of suffixes after the last occurence of some token in $Z_i$. Since semilinear languages are closed under concatenation, and $\Sigma^*$ and the finite language $Z_i$ are trivially semilinear, we just need to show that $L_i$ is semilinear.
Counter $i$ cannot be set to zero on strings of $L_i$, so we can write
\begin{align}
b_i = c_i(x) = \sum_{t=1}^n u_i(x_t)
= \sum_{\sigma \in \Sigma} u_i(\sigma) \#_\sigma(x)
= \bv u_i \cdot \Psi(x) \label{eq:counter-sum}
\end{align}
\noindent where $\bv u_i$ denotes the vector of possible updates to counter $i$ where each index corresponds to a different $\sigma \in \Sigma$. So, $L_i$ is the linear language
\begin{equation}
    L_i = \{ x \in \Sigma^* \mid \bv u_i \cdot \Psi(x) - b_i = 0 \} .
\end{equation}
\end{proof}

Although the proof of \autoref{thm:semilinearity} is nontrivial, $\nQSCL$ is a weak class. Such languages have limited ability to even detect the relative order of tokens in a string. We hope the proof might be extended to show $\SCL$ or $\CL$ is semilinear.

%% file: conclusion.tex
\section{Conclusion} \label{sec:conclusion}

We have shown that many variants of the counter machine converge to express the same class of formal languages, which supports that $\CL$ is a robustly defined class. The variations we explored move the classical general counter machine \cite{fischer1968counter} closer to the LSTM in form without changing its expressive power. We also proved real-time counter languages are closed under a large number of common set operations, providing tools for future work investigating counter automata.

We also showed that counter automata are incapable of evaluating boolean expressions, even though they are capable of verifying that boolean expressions are syntactically well-formed. This result has a clear parallel in the domain of natural language: deciding whether a sentence is grammatical is different than building a sentence's correct compositional meaning. A general take-away from our results is that just because a counter machine (or LSTM) is sensitive to surface patterns in language does not mean it can build correct semantic representations. Counter memory can be exploited to weakly match patterns in linguistic data, which might provide the wrong kinds of inductive bias for achieving sophisticated natural language understanding.

Finally, we asked whether counter languages are semilinear as another way of studying their power. We concluded that a weak subclass of the counter languages are semilinear, and encourage future work to address the general case.

%% file: main.bbl
\begin{thebibliography}{10}
\providecommand{\url}[1]{\texttt{#1}}
\providecommand{\urlprefix}{URL }
\providecommand{\doi}[1]{https://doi.org/#1}

\bibitem{bresnan-1982-cross-serial}
Bresnan, J., Kaplan, R.M., Peters, S., Zaenen, A.: Cross-serial dependencies in
  dutch. Linguistic Inquiry  \textbf{13}(4),  613--635 (1982),
  \url{http://www.jstor.org/stable/4178298}

\bibitem{chomsky1956three}
Chomsky, N.: Three models for the description of language. IRE Transactions on
  information theory  \textbf{2}(3),  113--124 (1956)

\bibitem{fischer1966turing}
Fischer, P.C.: Turing machines with restricted memory access. Information and
  Control  \textbf{9}(4),  364--379 (1966)

\bibitem{fischer1968counter}
Fischer, P.C., Meyer, A.R., Rosenberg, A.L.: Counter machines and counter
  languages. Mathematical systems theory  \textbf{2}(3),  265--283 (Sep 1968).
  \doi{10.1007/BF01694011}, \url{https://doi.org/10.1007/BF01694011}

\bibitem{hochreiter-1997-lstm}
Hochreiter, S., Schmidhuber, J.: Long short-term memory. Neural Computation
  \textbf{9}(8),  1735--1780 (1997). \doi{10.1162/neco.1997.9.8.1735},
  \url{https://doi.org/10.1162/neco.1997.9.8.1735}

\bibitem{joshi1990convergence}
Joshi, A.K., Shanker, K.V., Weir, D.: The convergence of mildly
  context-sensitive grammar formalisms. Technical Reports (CIS) p.~539 (1990)

\bibitem{linzen2016assessing}
Linzen, T., Dupoux, E., Goldberg, Y.: Assessing the ability of {LSTM}s to learn
  syntax-sensitive dependencies. Transactions of the Association for
  Computational Linguistics  \textbf{4},  521--535 (2016).
  \doi{10.1162/tacl\_a\_00115}, \url{https://www.aclweb.org/anthology/Q16-1037}

\bibitem{merrill-2019-sequential}
Merrill, W.: Sequential neural networks as automata. In: Proceedings of the
  Workshop on Deep Learning and Formal Languages: Building Bridges. pp. 1--13.
  Association for Computational Linguistics, Florence (Aug 2019),
  \url{https://www.aclweb.org/anthology/W19-3901}

\bibitem{Pullum1980-PULNLA}
Pullum, G.K., Gazdar, G.: Natural languages and context-free languages.
  Linguistics and Philosophy  \textbf{4}(4),  471--504 (1980).
  \doi{10.1007/BF00360802}

\bibitem{suzgun-etal-2019-lstm}
Suzgun, M., Belinkov, Y., Shieber, S., Gehrmann, S.: {LSTM} networks can
  perform dynamic counting. In: Proceedings of the Workshop on Deep Learning
  and Formal Languages: Building Bridges. pp. 44--54. Association for
  Computational Linguistics, Florence (Aug 2019),
  \url{https://www.aclweb.org/anthology/W19-3905}

\bibitem{suzgun2019memory}
Suzgun, M., Gehrmann, S., Belinkov, Y., Shieber, S.M.: Memory-augmented
  recurrent neural networks can learn generalized dyck languages (2019)

\bibitem{weiss2018}
Weiss, G., Goldberg, Y., Yahav, E.: On the practical computational power of
  finite precision {RNNs} for language recognition. CoRR
  \textbf{abs/1805.04908} (2018), \url{http://arxiv.org/abs/1805.04908}

\end{thebibliography}
